\documentclass[letterpaper, 10 pt, conference]{ieeeconf}  %

\IEEEoverridecommandlockouts                              %

\usepackage{algorithm,algorithmicx,algpseudocode}
\usepackage[cmex10]{amsmath} %
\usepackage{bm} %
\usepackage{mathtools}

\usepackage{amssymb,amsfonts,amscd,dsfont,mathrsfs,bbm, amsthm} %
 \usepackage[mathcal]{euscript} 
\usepackage{hyperref} %
\usepackage[pdftex]{graphicx}
\graphicspath{{../fig/}}
\DeclareGraphicsExtensions{.pdf}

\usepackage{color}

\usepackage{scalerel} %

\usepackage{tikz}
\usetikzlibrary{shapes,arrows}
\usetikzlibrary{calc}
\allowdisplaybreaks

\tikzset{%
  bullet/.style={yshift = -.5mm, scale = 1.3},
  extfill/.style={fill = blue!20}, 
  frozenfill/.style={fill = gray, fill opacity = .4}, 
  ext/.style={trapezium, trapezium angle=67.5, draw,
  inner ysep=5pt, outer sep=0pt, extfill,
  minimum height=1.2cm, minimum width=0pt},
  extl/.style={isosceles triangle,
    isosceles triangle apex angle=60, minimum height=1.2cm,
    draw, extfill, minimum height=1.2cm, minimum width=0pt},
  block/.style    = {draw, thick, rectangle, minimum height = 3em, minimum width = 5em},
  sum/.style      = {draw, circle, node distance = 2cm}, 
  con/.style      = {draw, circle, node distance = 2cm}, 
  input/.style    = {coordinate}, 
  int/.style    = {coordinate}, 
  missing/.style={
    draw=none,
    fill=none,
    yshift = 0.1cm,
    scale = 1.5,
    execute at begin node={\color{black}{$\vdots$}}
  },
  output/.style   = {coordinate}, 
  exts/.style={ext, 
  minimum height=1cm},
}

\newtheorem{theorem}{Theorem}
{Lemma}
\newtheorem{proposition}
{Proposition}
\newtheorem{corollary}
{Corollary}
\newtheorem{definition}
{Definition}
\newtheorem{remark}
{Remark}
\newtheorem{property}
{Property}
\newtheorem{fact}
{Fact}
{Assumption}
\newtheorem{example}{Example}

\newcommand{\ds}{\displaystyle}
\newcommand{\secref}[1]{Section~\ref{#1}}

\newcommand{\figref}[1]{Fig.~\ref{#1}}

\newcommand{\thmref}[1]{Theorem~\ref{#1}}
\newcommand{\propref}[1]{Proposition~\ref{#1}}
\newcommand{\proptyref}[1]{Property~\ref{#1}}
\newcommand{\corolref}[1]{Corollary~\ref{#1}}
\newcommand{\factref}[1]{Fact~\ref{#1}}
\newcommand{\defref}[1]{Definition~\ref{#1}}

\newcommand{\ev}{{\mathbb{E}}}
\newcommand{\pr}{\mathbb{P}}
\newcommand{\prob}[1]{\pr\left\{#1\right\}}

\newcommand{\E}[1]{\ev\left[#1\right]}
\newcommand{\bE}[1]{\ev\bigl[#1\bigr]}

\newcommand{\Ed}[2]{\ev_{#1}\left[#2\right]}

\newcommand{\ellh}{\ell_\hinge}
\newcommand{\hinge}{\mathrm{hinge}}

\renewcommand{\vec}[1]{\underline{#1}}

\newcommand{\cC}{\mathcal{C}}
\newcommand{\cD}{\mathcal{D}}

\newcommand{\cF}{\mathcal{F}}

\newcommand{\cI}{\mathcal{I}}

\newcommand{\cP}{\mathcal{P}}

\newcommand{\cS}{\mathcal{S}}
\newcommand{\cT}{\mathcal{T}}
\newcommand{\cX}{\mathcal{X}}
\newcommand{\cY}{\mathcal{Y}}
\newcommand{\cZ}{\mathcal{Z}}

\newcommand{\defeq}{\triangleq}

\newcommand{\T}{\mathrm{T}}

\newcommand{\norm}[1]{\|#1\|}
\newcommand{\bnorm}[1]{\bigl\|#1\bigr\|}
\newcommand{\bbnorm}[1]{\left\|#1\right\|}
\newcommand{\ip}[2]{\langle#1,#2\rangle}

\newcommand{\fip}[2]{\left\langle#1,#2\right\rangle}

\newcommand{\fproj}[2]{#1|_{#2}} 

\newcommand{\fh}{\hat{f}}
\newcommand{\gh}{\hat{g}}

\newcommand{\xh}{\hat{x}}
\newcommand{\yh}{{\hat{y}}}

\newcommand{\phih}{\hat{\phi}}
\newcommand{\psih}{\hat{\psi}}

\newcommand{\cXh}{\hat{\cX}}
\newcommand{\cYh}{\hat{\cY}}

\newcommand{\cXt}{\tilde{\cX}}
\newcommand{\cYt}{\tilde{\cY}}

\newcommand{\Sh}{\hat{S}}
\newcommand{\Th}{\hat{T}}

\newcommand{\Xh}{\hat{X}}
\newcommand{\Yh}{\hat{Y}}

\newcommand{\Xt}{\tilde{X}} 
\newcommand{\Yt}{\tilde{Y}} 
\newcommand{\Pt}{\tilde{P}}

\DeclareMathOperator{\diag}{diag}
\DeclareMathOperator{\corr}{cov}

\DeclareMathOperator{\var}{var}
\DeclareMathOperator{\tr}{tr}

\DeclareMathOperator*{\argmax}{arg\,max}
\DeclareMathOperator*{\argmin}{arg\,min}
\DeclareMathOperator*{\maximize}{maximize}
\DeclareMathOperator*{\minimize}{minimize}
\DeclareMathOperator*{\st}{subject~to}

\DeclareMathAlphabet{\mathbbb}{U}{bbold}{m}{n}  
\newcommand{\kron}{\mathbbb{1}}
\newcommand{\id}{\mathbb{I}}

\newcommand{\cW}{\mathcal{W}}

\newcommand{\funcs}{\cF}

\newcommand{\spc}[1]{\funcs_{#1}}
\newcommand{\spcn}[2]{\spc{#1}^{\,#2}}
\newcommand{\spcns}[1]{\spc{#1}^{\,*}} 

\newcommand{\Unif}{\mathrm{Unif}}

\newcommand{\fss}{f^*}
\newcommand{\gss}{g^*}
\newcommand{\gssT}{g^{*\T}}

\newcommand{\mcf}[1]{f^*_{#1}}
\newcommand{\mcg}[1]{g^*_{#1}}

\newcommand{\As}{\mathscr{A}} %
\newcommand{\alg}{\mathsf{alg}} %

\newcommand{\Hs}{\mathscr{H}} 
\newcommand{\Hnest}{\Hs_{\sf nest}}

\newcommand{\lpmi}{\mathfrak{i}} 

\newcommand{\mdn}[1]{\md_{#1}} 
\newcommand{\mdnl}[1]{\md_{\leq #1}} 
\DeclareMathOperator{\md}{\zeta} 
\DeclareMathOperator{\rank}{rank} 

\newcommand{\La}{\Lambda}

\newcommand{\xedit}[1]{\textcolor{red}{\##1}}

\newcommand\copyrighttext{%
  \footnotesize \textcopyright 2024 IEEE. Personal use of this material is permitted. Permission from IEEE must be obtained for all other uses, in any current or future media, including reprinting/republishing this material for advertising or promotional purposes, creating new collective works, for resale or redistribution to servers or lists, or reuse of any copyrighted component of this work in other works.

  IEEE Xplore link to the published article: \href{https://ieeexplore.ieee.org/document/10735284}{\textbf{DOI} 10.1109/Allerton63246.2024.10735284}.}
\newcommand\copyrightnotice{%
\begin{tikzpicture}[remember picture,overlay]
\node[anchor=south,yshift=2pt] at (current page.south) {\fbox{\parbox{\dimexpr\textwidth-\fboxsep-\fboxrule\relax}{\copyrighttext}}};
\end{tikzpicture}%
}

\newcommand{\optional}[1]{} 

\title{\LARGE \bf
  Dependence Induced Representations}

\author{Xiangxiang Xu$^{1}$
  and Lizhong Zheng$^{1}$%
  \thanks{*This work was supported in part by the Office of Naval Research (ONR) under grant N00014-19-1-2621.}%
  \thanks{$^{1}$Xiangxiang Xu and Lizhong Zheng are with the Department of Electrical Engineering and Computer Science, Massachusetts Institute of Technology, Cambridge, MA 02139  {\tt\small \{xuxx, lizhong\}@mit.edu}}%
}

\newcommand{\ly}{\mathtt{y}} %
\newcommand{\dname}{D-loss}
\newcommand{\dnames}{D-losses}
\newcommand{\Dname}{D-Loss}
\newcommand{\Dnames}{D-Losses}
\newcommand{\radj}{regular}

\newcommand{\rdname}{\radj~\dname}
\newcommand{\rdnames}{\radj~\dnames}

\begin{document}

\maketitle
\copyrightnotice

\thispagestyle{empty}
\pagestyle{empty}

\begin{abstract}
We study the problem of learning feature representations from a pair of random variables, where we focus on the representations that are induced by their dependence. We provide sufficient and necessary conditions for such dependence induced representations, and illustrate their connections to Hirschfeld--Gebelein--R\'{e}nyi (HGR) maximal correlation functions and minimal sufficient statistics. We characterize a large family of loss functions that can learn dependence induced representations, including cross entropy, hinge loss, and their regularized variants. In particular, we show that the features learned from this family can be expressed as the composition of a loss-dependent function and the maximal correlation function, which reveals a key connection between representations learned from different losses. Our development also gives a statistical interpretation of the neural collapse phenomenon observed in deep classifiers. Finally, we present the learning design based on the feature separation, which allows hyperparameter tuning during inference.
\end{abstract}

\section{Introduction}

Deep learning techniques \cite{lecun2015deep} have demonstrated significant progress in extracting useful information from high-dimensional structured data. The learned feature representations can be easily adapted to different learning objectives, data modalities, or learning tasks. Compared with the practical successes, the understanding of the learned representations is rather limited: though it is widely acknowledged that such representations contain useful information for understanding the statistical behaviors of data \cite{breiman2001statistical}, theoretical analyses of learned representations are still challenging. In particular, the difficulty comes from the unknown complicated structures behind high-dimensional data and the huge design space of practical learning algorithms. It is often difficult to obtain analytical solutions to the representations, except for a few cases with principled designs \cite{xu2024neural}. Theoretical analyses often rely on
additional assumptions, making the results restricted to specific scenarios.

In contrast to the difficulty faced by theoretical characterizations, deep learning practices suggest the other extreme: despite the huge design space of loss functions and hyperparameter settings, the learned representations can still be reused or adapted to different setups; empirical studies also indicate 
that representations obtained from completely different methods can be highly similar \cite{huhposition}. A better understanding of this similarity can provide
 key insights on the learning mechanism and lead to more efficient learning designs.

In this paper, we tackle the representation learning problem from a statistical perspective. We consider learning representations from a pair of variables $X, Y$ and focus on the representations determined only by the $X-Y$ dependence. We refer to such representations as dependence induced representations, formalized as the representations invariant to transformations that preserve the dependence structure. We provide sufficient and necessary conditions for such representations and demonstrate their deep connections to the Hirschfeld--Gebelein--R\'{e}nyi (HGR) maximal correlation functions and the notion of minimal sufficiency. Furthermore, we characterize a family of loss functions, termed \dnames, which can be used for learning dependence induced representations. Specifically, we demonstrate that the optimal feature is the composition of a loss-dependent term and the maximal correlation function, revealing the key connection between different feature representations.
We also show that \dnames{} cover a large class of practical loss functions, including cross entropy and hinge loss as special cases.
Our characterization also provides a statistical interpretation of the neural collapse behavior \cite{papyan2020prevalence} in training deep classifiers. Based on our developments, we propose a design of feature adapters that allow hyperparameter tuning during inference.

\section{Notations and Preliminaries}
For a random variable $Z$, we use $\cZ$ to denote the corresponding alphabet and use $z$ to denote a specific value in $\cZ$. We use $P_Z$ to denote the probability distribution\footnote{Throughout our development, we restrict to discrete random variables on finite alphabets with positive probability masses, i.e., $P_Z(z) > 0$ for all $z \in \cZ$.}
of $Z$ and denote the collection of probability distributions supported on $\cZ$ by $\cP^\cZ$. For $n \geq 1$, we denote $[n] \defeq \{1, 2, \dots, n\}$.

\subsection{Feature Space and Representation Learning} 

Given data $Z$ with domain $\cZ$, we refer to the mappings from $\cZ$ to vector space as feature functions or feature mappings and refer to the mapped values as feature representations (or simply, features). In particular, we use $\spc{\cZ} \defeq \{ \cZ \to \mathbb{R}\}$ to denote the collection of one-dimensional feature functions, and denote the collection of $k$-dimensional features by $\spcn{\cZ}{k} \defeq \left(\spc{\cZ}\right)^k = \{ \cZ \to \mathbb{R}^k\}$ for each $k \geq 1$. We use $\spcns{\cZ} \defeq \bigcup_{k \geq 1}\spcn{\cZ}{k}$ to denote all such feature functions. \optional{For given $f \in \spcns{\cZ}$, we use $f[\cZ] \defeq \{f(z) \colon z \in \cZ\}$ to denote its image.} For $f, h \in \spcns{\cZ}$, we denote $\La_{f, h} \defeq \E{f(Z)h^\T(Z)}$. Specifically, we use $\La_f$ to represent $\La_{f, f} = \E{f(Z)f^\T(Z)}$.

\subsubsection{Abstractions; Informational Equivalence}
Given two random variables $X$ and $Y$, following the convention of  \cite{shannon1953lattice}, we call $Y$ an \emph{abstraction} of $X$ if there exists a function $h\colon \cX \to \cY$ such that $Y = h(X)$. If each of $X, Y$ is an abstraction of the other,  we say that $X$ and $Y$ are  \emph{informationally equivalent}\footnote{The concept of informational equivalence was also introduced independently in later psychological studies, e.g., \cite{Simon1978-SIMOTF-5}.}.

By definition, the conditional expectation $\E{Y|X}$ is an abstraction of $X$. Specifically, given $f \in \spcns{\cZ}$ and $s \in \spcns{\cZ}$, $\fproj{f}{s}(Z)$ is an abstraction of $s(Z)$, where we have defined %
\begin{align}
  \fproj{f}{s}(z) \defeq \E{f(Z)|s(Z) = s(z)}, \qquad \text{for all }z \in \cZ.
  \label{eq:fproj:def}
\end{align}

\optional{The $\fproj{f}{s}$ corresponds to the projection of $f$ onto the space of all functions of $s(Z)$.}

\subsubsection{Representation Learning Algorithms} 
We restrict to the representation learning involving a pair of random variables $(X, Y) \sim P_{X, Y}$.%
We refer to each mapping of the form  
\begin{align}
  (X, Y) \overset{}{\mapsto}  (f, g)
  \label{eq:rep:alg}
\end{align}
as a representation learning algorithm: it takes the random variable pair $(X, Y)$ as the input and outputs 
feature function pair $f \in \spcn{\cX}{k}$ and $g \in \spcn{\cY}{k}$ for a given\footnote{Though the $k$ is required to make the mapping \eqref{eq:rep:alg} well-defined, we will only specify the choice when necessary.}
 $k \geq 1$, which map original data $X, Y$ to the $k$-dimensional feature representations $f(X), g(Y) \in \mathbb{R}^k$, respectively. %
Note that this is a mathematical simplification of the practical learning algorithms: we focus on the input-output relation, and omit the implementation details such as the computation. 

\optional{remark: single sided version as a special case, only output one of $f$ or $g$}

\subsection{Canonical Dependence Kernel and Maximal Correlation}%
Given $(X, Y) \sim P_{X, Y}$, we define the associated \emph{canonical dependence kernel} (CDK) function \cite{HuangMWZ2024} as a joint function $\lpmi_{X; Y} \in \spc{\cX \times \cY}$, with
\begin{align}
  \lpmi_{X; Y}(x, y) \defeq \frac{P_{X, Y}(x, y)}{P_{X}(y)P_Y(y)} - 1.
  \label{eq:cdk:def}
\end{align}

We can apply the modal decomposition \cite{xu2024neural}, a singular value decomposition (SVD) in function space, to write the CDK as a superposition of rank-one singular modes:
\begin{align}
  \lpmi_{X; Y}(x, y) = \sum_{i = 1}^K  \sigma_i \mcf{i}(x) \mcg{i}(y),
  \label{eq:cdk:md}
\end{align}
where we have the singular values $\sigma_1 \geq \sigma_2 \geq \dots \geq \sigma_K > 0$, and where $K \geq 0$ denotes the rank of $\lpmi_{X; Y}$. We also have the orthogonality relations of functions $\mcf{i}, \mcg{i}$: $\E{\mcf{i}(X) \mcf{j} (X)} = \E{\mcg{i}(Y) \mcg{j}(Y)} = \delta_{i, j}$, for all $ i, j \in [K]$. %

The functions $\mcf{i}, \mcg{i}$ correspond to the maximally correlated functions in $\spc{\cX}$ and $\spc{\cY}$, known as Hirschfeld--Gebelein--R\'{e}nyi (HGR) maximal correlation functions \cite{hirschfeld1935connection, gebelein1941statistische, renyi1959measures}. To see the connection, let us denote the covariance of $f \in \spc{\cX}$ and $g \in \spc{\cY}$ as
  \begin{align*}
    \corr(f, g) \defeq \Ed{P_{X, Y}}{f(X)g(Y)} - \Ed{P_{X}P_{Y}}{f(X)g(Y)}.
  \end{align*}
  Then, for all $i = 1, \dots, K$, we have $\sigma_i = \corr(\mcf{i}, \mcg{i}) = \Ed{P_{X, Y}}{\mcf{i}(X)\mcg{i}(Y)}$ and  $\ds (\mcf{i}, \mcg{i}) = \argmax_{f_i, g_i}\, \corr(f_i, g_i)$,
  where the maximization is taken over all $f_i \in \spc{\cX}$ and $g_i \in \spc{\cY}$ such that $\E{f_i^2(X)} = \E{g_i^2(Y)} = 1$ and $\E{f_i(X)\mcf{j}(X)} = \E{g_i(Y)\mcg{j}(Y)} = 0$ for all $j < i$.

  For convenience, we also define 
\begin{align}
\fss \defeq (\mcf{1}, \dots, \mcf{K})^\T, \quad \gss \defeq (\mcg{1}, \dots, \mcg{K})^\T.\label{eq:fss:gss}  
\end{align}%

\subsection{Sufficient Statistics and Minimal Sufficiency}%
We start with the definition of sufficiency and minimal sufficiency. For more detailed discussions, see, e.g.,  
 \cite[Section 2.9]{cover2006elements}.
\begin{definition}
 Given $X, Y$, $f(X)$ is a sufficient statistic of $X$ for inferring $Y$ if we have the Markov relation $X-f(X)-Y$. We call $f(X)$ a minimal sufficient statistic if $f(X)$ is sufficient, and for each given sufficient statistic $T$ of $X$, $f(X)$ is an abstraction of $T$.
\end{definition}

By definition, all minimal sufficient statistics are informationally equivalent. Therefore, it suffices to specify one representative among the equivalent class, say, $S = s(X)$ from some $s$. Then, we can characterize the collection of all sufficient statistics as $\{f(X)\colon \text{$S$ is an abstraction of $f(X)$}\}$.

We can also express the sufficiency in terms of CDK functions. Specifically, let $\lpmi_{X; Y}(X, Y)$ indicate the random variable $\gamma(X, Y)$ where $\gamma = \lpmi_{X; Y}$. Then we have the following property. The proof is omitted.
\begin{property}
  $f(X)$ is a sufficient statistic if and only if $\lpmi_{X; Y}(X, Y)$ is an abstraction of $(f(X), Y)$.
\end{property}

\optional{
\begin{proof}
  \xedit{..}  
\end{proof}
}
By symmetry, $g(Y)$ is sufficient if and only if $\lpmi_{X; Y}(X, Y)$ is an abstraction of $(X, g(Y))$. This motivates us to consider a symmetric notion of sufficiency: we say that $f(X)$ and $g(Y)$ are \emph{jointly sufficient} if $\lpmi_{X; Y}(X, Y)$ is an abstraction of $(f(X), g(Y))$. 
In particular, we have the following characterization. The proof is omitted.

\begin{property}
\label{propty:joint:ss}
The following statements are equivalent:
\begin{itemize}
\item $f(X)$ and $g(Y)$ are jointly sufficient;
\item Both $f(X)$ and $g(Y)$ are sufficient statistics;
\item We have the Markov chain $X-f(X)-g(Y)-Y$.
\end{itemize}  
\end{property}
\optional{
\begin{proof}
  \xedit{start here}
\end{proof}
}

\begin{remark}
  We can replace $\lpmi_{X; Y}(X, Y)$ in these statements by its informationally equivalent transformations. For example, we can obtain the pointwise mutual information $ \log(1 + \lpmi_{X; Y}(x, y))  = \left[\log\frac{P_{X, Y}}{P_XP_Y}\right](x, y)$ by considering the mapping $t \mapsto \log(1 + t)$.
\end{remark}

\begin{remark}
  We can have equivalent results expressed in terms of information measures. Specifically, the abstraction relation and informational equivalence were originally characterized by using conditional entropies \cite{shannon1953lattice}, and the sufficiency can be characterized as some mutual information equalities \cite[Section 2.9]{cover2006elements}.
\end{remark}

\section{Definition and Fundamental Properties}

\subsection{Dependence Preserving Transformations}

We first formalize \emph{dependence preserving transformations} as follows.

\begin{definition}
\label{def:dep}
Given $(X, Y) \sim P_{X, Y}$, a \emph{dependence preserving transformation} 
generates transformed variables $(\Xh, \Yh)$, where $\Xh = \xi(X, Z)$ and  $\Yh = \eta(Y, W)$ are characterized by
\begin{itemize}
\item transformed alphabets $\cXh, \cYh$ of $\Xh$ and $\Yh$, respectively; 
\item random sources $Z \in \cZ$, $W \in \cW$ satisfying the Markov chain $Z - X - Y - W$, characterized by %
  conditional distributions $P_{Z|X}, P_{W|Y}$;
\item deterministic mappings $\xi \colon \cX \times \cZ \to \cXh$, $\eta \colon \cY \times \cW \to \cYh$ such that $X$ and $Y$ are abstractions of $\Xh$, $\Yh$, respectively.\footnote{This can be achieved when we have
    \begin{align*}
      P_{\Xh|X}(\xh|x)P_{\Xh|X}(\xh|x') &= 0 \text{ for all $\xh \in \cXh$ and $x, x' \in \cX$ with $x \neq x'$},\\
      P_{\Yh|Y}(\yh|y)P_{\Yh|Y}(\yh|y') &= 0 \text{ for all $\yh \in \cYh$ and $y, y' \in \cY$ with $ y\neq y'$}.
    \end{align*}
}
\end{itemize}
\end{definition}

With a slight abuse of notation, we use $\xi^{-1}$ and $\eta^{-1}$ to denote the mappings that decode $X, Y$ from $\Xh$, $\Yh$, respectively. This gives $X = \xi^{-1}(\Xh)$ and $Y = \eta^{-1}(\Yh)$.%

\begin{remark}
  A degenerated case of dependence preserving transformations is the one-to-one deterministic mapping, which can be obtained by setting both $Z$ and $W$ to constants.
\end{remark}

\optional{why we must allow redundant; counterexample,  $\max P_X(x)$}

\optional{$Z, W$: stronger condition $P_{Z, W|X, Y} = P_Z P_W$}

To see why the dependence is preserved, note that we can obtain $\Xh - X - Y - \Yh$ from the Markov relation $Z-X-Y-W$, i.e., the transformations do not introduce additional information. Moreover, there is no loss of information: we have $X - \Xh - \Yh - Y$ as $X$ and $Y$ are uniquely determined from $\Xh$ and $\Yh$, respectively. We can formalize this dependence preserving property in terms of the CDK as follows. The proof is omitted.

\begin{proposition}
\label{prop:cdk:eq}
Given $(X, Y)$, let $\Xh \in \cXh, \Yh \in \cYh$ be the variables obtained  from dependence preserving transformations $\xi, \eta$ [cf. \defref{def:dep}]. Then, for all $(\xh, \yh) \in \cXh \times \cYh$, we have $\lpmi_{\Xh; \Yh}(\xh, \yh) = \lpmi_{X; Y}(\xi^{-1}(\xh), \eta^{-1}(\yh))$.
\end{proposition}

\optional{| this is indeed a preorder: reflective \& transitive}

\subsection{Invariance, Maximal Correlation \& Minimal Sufficiency}

We first introduce the following definition.

\begin{definition}
  \label{def:di}%
  Given $(X, Y)$, we say a mathematical object\footnote{
This includes scalars, matrices, functions, random variables, etc.
} $\theta(X, Y)$ defined on $(X, Y)$ is \emph{dependence induced} if it is invariant to dependence preserving transformations, i.e., $\prob{\theta(X, Y) = \theta(\Xh, \Yh)} = 1$, for all $(\Xh, \Yh)$ obtained from $(X, Y)$ through dependence preserving transformations.%
\end{definition}
Therefore, being dependence induced implies the object relies only on the dependence structure, regardless of the irrelevant information introduced during dependence preserving transformations, such as the $Z, W$ in \defref{def:dep}.

\begin{theorem}
  \label{thm:di}
  Given $(X, Y)$, $(\fss(X), \gss(Y))$ is dependence induced, and we can 
construct a dependence preserving transformation to transform $(\fss(X), \gss(Y))$ into $(X, Y)$, %
where $\fss, \gss$ are the maximal correlation functions [cf. \eqref{eq:fss:gss}].
\end{theorem}

\begin{proof}
We first verify that $(f^*(X), g^*(Y))$ is dependence induced. To see this, we assume the $\Xh, \Yh$ are transformed from $X, Y$, respectively, where $\Xh = \xi^{-1}(X), \Yh = \eta^{-1}(Y)$ with probability one. %
Suppose $\lpmi_{X; Y}$ has the modal decomposition \eqref{eq:cdk:md}. Then, it follows from \propref{prop:cdk:eq} that, for all $(\xh, \yh) \in \cXh \times \cYh$,
\begin{align*}
  \lpmi_{\Xh; \Yh}(\xh, \yh)
  &= \lpmi_{X; Y}(\xi^{-1}(\xh), \eta^{-1}(\yh))\\
  &= \sum_{i \in [K]} \sigma_i \cdot \mcf{i} (\xi^{-1}(\xh)) \cdot \mcg{i} (\eta^{-1}(\yh))\\
  &= \sum_{i \in [K]} \sigma_i \cdot (\mcf{i} \circ \xi^{-1}) (\xh) \cdot  (\mcg{i} \circ \eta^{-1})(\yh),
\end{align*}
which is the modal decomposition of $\lpmi_{\Xh; \Yh}$. Therefore, the maximal correlation functions for $\Xh, \Yh$ are $\fss \circ \xi^{-1}$, $\gss \circ \eta^{-1}$. This implies
\begin{align*}
&\prob{( \fss \circ \xi^{-1})(\Xh) = \fss(X)}\\
 &\qquad= \prob{( \gss \circ \eta^{-1})(\Yh) = \gss(Y)} = 1.
\end{align*}
As a result, it holds with probability one that
\begin{align*}
\left(( \fss \circ \xi^{-1})(\Xh), ( \gss \circ \eta^{-1})(\Yh)
\right) = \left(\fss(X), \gss(Y)\right).
\end{align*}

Let $S = \fss(X), T = \gss(Y)$, we then construct a dependence preserving transformation that transforms $(S, T)$ into $(\Sh, \Th) = (X, Y)$ as follows. Due to the Markov relation $X - S - T - Y$, we can use  $Z = X, W = Y$ as the random sources [cf. \defref{def:dep}], and set the mappings to $\xi\colon (s, z) \mapsto z$ and $\eta \colon (t, w) \mapsto w$. It can be verified that this transformation satisfies all conditions in \defref{def:dep} and preserves the dependence.
\end{proof}

\begin{remark}
  Though the mapped results $\fss(X)$ and $\gss(Y)$ are dependence induced, the feature mappings $\fss, \gss$ need to adapt to the dependence preserving transformations and are \emph{not} dependence induced.
\end{remark}

We have the following corollary.

\begin{corollary}%
  \label{cor:di}
  A mathematical object $\theta(X, Y)$ defined on $(X, Y)$  is \emph{dependence induced} if and only if $\theta(X, Y) = \theta(\fss(X), \gss(Y))$.
\end{corollary}
\begin{proof}
  Suppose $(\Xh, \Yh)$ is obtained from $(X, Y)$ by applying a dependence preserving transformation. To see the sufficiency, note that from \thmref{thm:di}, $(\fss(X), \gss(Y))$ is dependence induced, and thus $\theta(\Xh, \Yh) = \theta(\fss(X), \gss(Y))$. The necessity follows from the fact that $(X, Y)$ can be transformed from $(\fss(X), \gss(Y))$ with a dependence preserving transformation [cf. \thmref{thm:di}].
\end{proof}
\optional{dependence induced quantities: MI, IB-curve, wyner?...}

\begin{example}
  the mutual information $I(X; Y) = I(\fss(X); \gss(Y))$ and
the matrix  $\diag(\sigma_1,\dots, \sigma_K)= \E{\fss(X)\gssT(Y)}$ are dependence induced.
\end{example}
\optional{lossy counterpart: DPI; ``dependence measures''}

From the corollary, dependence induced mathematical objectives do not contain more information than $(\fss(X), \gss(Y))$. In addition, since $(\fss(X), \gss(Y))$ is dependence induced, we can characterize such objectives by considering only the input $(\fss(X), \gss(Y))$. This gives a canonical form of $\theta$ that is invariant to dependence preserving transformations.

\optional{deeper connection between invariance and minimal sufficiency: this is not a coincidence }

\thmref{thm:di} and the corollary also suggest the fundamental role of $\fss(X), \gss(Y)$ (and their informationally equivalent transformations) in characterizing the $X-Y$ dependence. It can be proved that $\fss(X), \gss(Y)$ are sufficient statistics, see, e.g., \cite[Section 2]{HuangMWZ2024}. We can also easily verify it with the notion of joint sufficiency: from \eqref{eq:cdk:md}, $\fss(X), \gss(Y)$ are joint sufficient, and are thus sufficient statistics [cf. \proptyref{propty:joint:ss}]. Moreover, we can demonstrate that they are indeed minimal sufficient statistics.

\begin{proposition}
  \label{prop:ss}
  Given $(X, Y)$, the maximal correlation functions $\fss(X)$, $\gss(Y)$ [cf. \ref{eq:fss:gss}] are minimal sufficient statistics.
\end{proposition}

\begin{proof}
  By symmetry, it suffices to establish the minimal sufficiency of $\fss(X)$. As $\fss(X)$ is a sufficient statistic, it remains only to prove the minimality. From the modal decomposition \eqref{eq:cdk:md}, we have
 $\mcf{i}(x) = \sigma_i^{-1} \E{\mcg{i}(Y)|X = x}$ for all $i \in [K]$. Therefore, for any given sufficient statistic $S = s(X)$, we have
  \begin{align*}
    \fss(x) = \Sigma^{-1} \E{\gss(Y)|X = x} &= \Sigma^{-1}\E{\gss(Y)|S = s(x)},%
  \end{align*}
where $\Sigma \defeq \diag(\sigma_1, \dots, \sigma_K)$.
 To obtain the second equality, we have used the fact that $P_{Y|X = x}$ is a function of $s(x)$, since $s(X)$ is sufficient. Therefore, $\fss(X)$ is an abstraction of any sufficient statistic $S$, which completes the proof.
\end{proof}

\optional{two additional proofs; singular values; nested H}

\subsection{Dependence Induced Representations and Learning}

We then consider the impacts of dependence preserving transformations on representation learning. Specifically, we focus on the learned representations that are dependence induced, i.e., invariant to such transformations. Formally, Let $\alg \colon (X, Y)  \overset{}{\mapsto} (f, g)$ be a representation learning algorithm that learns $k$-dimensional feature functions $f \in \spcn{\cX}{k}, g \in \spcn{\cY}{k}$ from $(X, Y) \sim P_{X, Y}$. Suppose $(\Xh, \Yh)$ are obtained by applying a dependence preserving transformation on $(X, Y)$, via $\Xh = \xi(X, Z), \Yh = \eta(Y, W)$, where $Z, W, \xi, \eta$ satisfy the conditions in \defref{def:dep}. Applying the same algorithm to the transformed data gives   feature functions
\begin{align*}
  (f^{(\xi)}, g^{(\eta)}) =  \alg(\Xh, \Yh),
\end{align*}
where we have denoted the learned feature functions by $f^{(\xi)} \in \spcn{\cXh}{k}$ and $g^{(\eta)}  \in \spcn{\cYh}{k}$, with $\cXh, \cYh$ denoting transformed alphabets. Note that the transformations $\xi, \eta$ are not separately input to the algorithm. 

From \defref{def:di}, the learned representations are dependence induced if 
\begin{align*}
  \prob{f^{(\xi)}(\Xh) = f(X), g^{(\eta)}(\Yh) = g(Y)} = 1
\end{align*}
 for all possible generating processes of $\Xh$, $\Yh$. We refer to the algorithm $\alg$ as a \emph{dependence learning algorithm} if it learns dependence induced representations. Specifically, we have the following characterization of dependence induced representations and corresponding learning algorithms. %

\begin{theorem}
\label{thm:sep:alg}
The collection $\As$ of dependence learning algorithms is given by 
\begin{align}
\As \defeq \{\alg \colon
\alg(X, Y)&= (\phi \circ \fss, \psi \circ \gss), \notag\\
                  (\phi, \psi) &= \alg(\fss(X), \gss(Y))
\},\label{eq:As}  
\end{align}
where $\fss, \gss$ are the maximal correlation functions [cf. \eqref{eq:fss:gss}].%
\end{theorem}

\begin{proof}
  We first verify that each $\alg \in \As$ for $\As$ as defined in \eqref{eq:As} is a dependence learning algorithm. It suffices to note that the representation pair learned from $\alg$ is $(\phi(\fss(X)), \psi(\gss(Y)))$. This is dependence induced since both $(\fss(X), \gss(Y))$ and the mapping pair $(\phi, \psi)$ are dependence induced [cf. \thmref{thm:di},  \corolref{cor:di} and \eqref{eq:As}].

To see the reverse direction, suppose $\alg$ is a dependence learning algorithm, and let $(\phi, \psi) = \alg(\fss(X), \gss(Y))$,  $(\fh, \gh) = \alg(X, Y)$ denote the learned feature functions. From \thmref{thm:di}, we can obtain 
$(X, Y)$ from $(\fss(X), \gss(Y))$ by applying a dependence preserving transformation. As the representations are invariant to this transformation, we have $\phi(\fss(X)) = \fh(X)$ and $\psi(\gss(Y)) = \gh(Y)$ with probability one. This gives $\fh = \phi \circ \fss, \gh = \psi \circ \gss$, which implies that $\alg \in \As$.
\end{proof}

From \thmref{thm:sep:alg}, each dependence learning algorithm is uniquely specified by the function pair $(\phi, \psi)$. This gives a canonical form of the algorithm since the mappings $\phi, \psi$ are dependence induced (cf. \eqref{eq:As}). The composition structure $(\phi \circ \fss, \psi \circ \gss)$ allows separate learning implementations of $(\phi, \psi)$ and $(\fss, \gss)$, as we will detail in later discussions.

Moreover, all dependence induced representations are abstractions of $\fss(X)$, $\gss(Y)$, which are minimal sufficient statistics [cf. \propref{prop:ss}]. As a result, the dependence learning algorithms extract only dependence-related information. From the definition of minimal sufficiency, we can readily obtain the following corollary of \thmref{thm:sep:alg}. The proof is omitted.

\begin{corollary}
  \label{cor:ss}
  Given $(X, Y)$ and a dependence learning algorithm $\alg \in \As$, suppose $(f, g) = \alg(X, Y)$ are the output feature functions. Then, %
  \begin{itemize}
  \item For any sufficient statistics $S = s(X)$ and $T = t(Y)$, $f(X)$, $g(Y)$ are abstractions of $S$ and $T$, respectively;
  \item If $f(X)$ and $g(Y)$ are jointly sufficient, then $f(X)$ and $g(Y)$ are minimal sufficient statistics.
  \end{itemize}
\end{corollary}

In particular, the second claim of \corolref{cor:ss} provides a way to test the minimal sufficiency by constructing corresponding dependence learning algorithms. 
An example is the following generalization of \propref{prop:ss}.

\begin{proposition}
  \label{prop:ss:gen}
  Given $(X, Y)$ and a one-to-one mapping $\tau \colon [-1, \infty) \to \mathbb{R}$, suppose $f \in \spcn{\cX}{k}$ and $g \in \spcn{\cY}{k}$ factorize $\tau(\lpmi_{X; Y})$, i.e., $\tau(\lpmi_{X; Y}(x, y)) = f^\T(x) g(y)$ for all $x \in \cX, y \in \cY$, and $\rank(\La_f) = \rank(\La_g) = k$.
  Then, $f(X)$ and $g(Y)$ are minimal sufficient statistics.
\end{proposition}
\optional{use $\rank(\La_f)=\rank(\La_g) = k$ instead}

\optional{add that PMI factorization is mss.}

\begin{proof}
  Given $(X, Y)$, consider the representation learning algorithm that outputs the solution to the optimization problem
  \begin{align*}
    \minimize_{f \in \spcn{\cX}{k}, g \in \spcn{\cY}{k}} \Ed{P_XP_Y}{\left|\tau(\lpmi_{X; Y}(X, Y)) - f^\T(X)g(Y)\right|^2}.
  \end{align*}
  It can be easily verified that the algorithm is a dependence learning algorithm. The solution is any $f, g$ that factorize $\tau(\lpmi_{X; Y})$, i.e., 
  \begin{align}
    f^{\T}(X)g(Y) = (\tau \circ \lpmi_{X; Y})(X, Y).\label{eq:factorize}    
  \end{align}
  Since $\rank(\La_f) = \rank(\La_g) = k$,
  the choice of $(f, g)$ is unique up to some invertible linear transformation pairs. In addition, \eqref{eq:factorize} implies that $f(X)$ and $g(Y)$ are jointly sufficient. Finally, the minimal sufficiency follows from \corolref{cor:ss}.
\end{proof}

We conclude this section by considering two extreme cases: independence and strict dependence, summarized as the following corollary. The proof is omitted.

\begin{corollary}
  \label{cor:ss:extreme}
  Given $(X, Y)$ and $\alg \in \As$, let $(f, g) = \alg(X, Y)$ be the learned feature functions. We have %
  \begin{enumerate}
  \item if $X$ and $Y$ are independent, i.e., $P_{X, Y} = P_XP_Y$, then $f(X)$ and $g(Y)$ are constants and are independent of the marginal distributions $P_X$ and $P_Y$; 
  \item if $Y$ is an abstraction of $X$, then $f(X)$ is an abstraction of $Y$.
  \end{enumerate}
\end{corollary}

\begin{remark}
  From our developments, we can also obtain results about information measures. For example, it follows from \thmref{thm:sep:alg} that the dependence induced representations $f(X), g(Y)$ satisfy $H(f(X)) \leq H(\fss(X))$ and $H(g(Y))\leq H(\gss(Y))$, where $H(\cdot)$ denotes the entropy.
\end{remark}

\optional{  
If there exists a function  $\xi \colon \cY \to \cX$ such that $Y = \xi(X)$, then $\xi(X)$ is a minimal sufficient statistic for $X$ to infer $Y$.}

\section{\Dnames{} and Dependence Learning }

Our discussions in the previous section have characterized the theoretical behaviors of dependence learning algorithms. It is, however, difficult to directly apply these results to practical deep representation learning, e.g., determining if a representation learning algorithm is a dependence learning algorithm or constructing dependence learning algorithms. To address these problems, it is necessary also to consider practical implementation and develop operable characterizations. To this end, we consider representation learning algorithms characterized by corresponding loss functions and characterize sufficient conditions on losses for constructing dependence learning algorithms. We investigate a family of loss functions, termed \dnames, and demonstrate their deep connections to learning practices.%

\subsection{\Dname{} Family}%
 
Given $(X, Y) \sim P_{X, Y}$, let $\Gamma = \Gamma(f, g; P_{X, Y})$ denote a functional defined on $(f, g) \in \spcn{\cX}{k} \times \spcn{\cY}{k}$. We will omit  $P_{X, Y}$ and write $\Gamma(f, g)$ if the choice of $P_{X, Y}$ is clear from the context. Specifically, when we evaluate $\Gamma$ on a dataset 
$\{(x_i, y_i)\}_{i = 1}^n$, the $P_{X, Y}$ corresponds to the empirical distribution $\hat{P}_{X, Y}(x, y) \defeq \frac{1}{n} \sum_{i = 1}^n \kron_{\{x = x_i, y = y_i\}}$, where $\kron_{\{\cdot\}}$ denotes the 0-1 indicator function. 

We define the \dnames{} as follows. %

\begin{definition}[\dname]
\label{def:cont}
Given $(X, Y) \sim P_{X, Y}$, a functional $\Gamma$ defined on $(f, g) \in \spcn{\cX}{k} \times \spcn{\cY}{k}$ is a \dname{} if 
\begin{enumerate}
 \item suppose $\phi \circ \xi \in \spcn{\cX}{k}$ and $\psi \circ \eta \in \spcn{\cY}{k}$, 
 then we have%
 \begin{align}
  \Gamma(\phi \circ \xi, \psi \circ \eta; P_{X, Y}) = \Gamma(\phi, \psi; P_{\xi(X), \eta(Y)});
  \label{eq:trans} 
\end{align}
\item for all $s \in \spcns{\cX}$ and $t \in \spcns{\cY}$ that satisfy the Markov relation $X-s(X)-t(Y)-Y$, we have %
  \begin{align}
       \Gamma(\fproj{f}{s}, \fproj{g}{t}) \leq \Gamma(f, g),\label{eq:d:ieq}
  \end{align}
  where $\fproj{f}{s}$ and $\fproj{g}{t}$ are as defined in \eqref{eq:fproj:def}.
\end{enumerate}
\end{definition}
Specifically, we say a \dname{} $\Gamma$ is \radj{}, 
if in \eqref{eq:d:ieq}, $\Gamma(f, g) = \Gamma(\fproj{f}{s}, \fproj{g}{t}) < \infty$ 
implies $(f, g) = (\fproj{f}{s}, \fproj{g}{t})$. 

We can obtain dependence learning algorithms from the \dname{} family, demonstrated as follows.

\begin{theorem}
  \label{thm:sep:loss}
  Given $(X, Y) \sim P_{X, Y}$, suppose $L$ defined on $(f, g) \in \spcn{\cX}{k} \times \spcn{\cY}{k}$ is a \dname{} and $v_k(X, Y) \defeq \min_{(f, g) \in \spcn{\cX}{k} \times \spcn{\cY}{k}}L(f, g; P_{X, Y})$ exists. Then,
  \begin{enumerate} %
  \item $v_k(X, Y)$ is dependence induced;
  \item %
    there exist dependence induced mappings $\phi, \psi$, such that $(f, g) = (\phi\circ \fss, \psi \circ \gss)$ achieves the minimum
    value $v_k$, where $\fss, \gss$  are as defined in \eqref{eq:fss:gss};%
  \item if $L$ is \radj{}, then the representation learning
    algorithm $\left((X, Y) \mapsto \argmin_{f, g}L(f, g; P_{X,Y})\right)$ is a dependence learning algorithm.
  \end{enumerate}
\end{theorem}

\optional{the minimum value as a function of $X, Y$, $k$, is dependence induced}

\begin{proof}
  We begin with the first two claims. Let $S = \fss(X), T = \gss(Y)$, and we use $\cS, \cT$ denote the corresponding alphabets. Suppose $(\fh, \gh)  \in \spcn{\cX}{k} \times \spcn{\cY}{k}$ achieves the minimum value $v_k(X, Y)$, and let $(\phih, \psih) \in \spcn{\cS}{k}\times \spcn{\cT}{k}$ be the mappings such that $\fproj{\fh}{\fss} = \phih \circ \fss, \fproj{\gh}{\gss} = \psih \circ \gss$. Then, from the Markov relation
$X - S - T - Y$ and \eqref{eq:d:ieq}, we obtain
\begin{align*}
  L(\fh, \gh; P_{X, Y}) &\geq L(\fproj{\fh}{\fss}, \fproj{\gh}{\gss}; P_{X, Y})\\
               &= %
                 L(\phih \circ \fss, \psih \circ \gss; P_{X, Y})%
                 = L(\phih, \psih; P_{S, T}),
\end{align*}
where the last equality follows from \eqref{eq:trans}. Note that since $(\fh, \gh)$ achieves the minimum value of $L$, the inequality must hold with equality, so the minimum value $v_k$ can be achieved by $(f, g) = (\phih \circ \fss, \psih \circ \gss)$. In addition, we have 
\begin{align}
  (\phih, \psih) \in \argmin_{\phi \in \spcn{\cS}{k}, \psi \in \spcn{\cT}{k}}L(\phi, \psi; P_{S, T}),
  \label{eq:phi:psi}
\end{align}
since otherwise we get an $L$ value smaller than $L(\fh, \gh; P_{X; Y})$. From \corolref{cor:di}, both $v_k$ and $(\phih, \psih)$ and dependence induced.
Finally, when $L$ is \radj, the equality $  L(\fh, \gh) = L(\fproj{\fh}{\fss}, \fproj{\gh}{\gss})$ implies $(\fh, \gh) = (\phih \circ \fss, \psih \circ \gss)$. Therefore, it follows from 
\eqref{eq:phi:psi} and \thmref{thm:sep:alg} that $\left((X, Y) \mapsto \argmin_{f, g}L(f, g; P_{X,Y})\right)$ is a dependence learning algorithm.
\end{proof}

Combining \thmref{thm:sep:loss} and \thmref{thm:sep:alg}, if $L$ is a \rdname{}, it suffices to solve $\minimize_{\phi, \psi}L(\phi, \psi; P_{S, T})$ instead of the original problem $\minimize_{f, g}L(f, g; P_{X, Y})$, where we have defined $S = \fss(X) , T = \gss(Y)$. In particular, the optimal solution takes a composition form $(f, g) = (\phi \circ \fss, \psi \circ \gss)$, where only $\phi, \psi$ depend on the form of $L$. This separation between loss-dependent functions $(\phi, \psi)$ and loss-invariant $(\fss, \gss)$ makes it possible to design efficient implementation for \dname{} minimization problems, which we will detail in the next section. 

Comparing the first claim with the second claim, a difference between \dnames{} and \rdnames{} is that minimizing \dnames{} can give solutions not in the form $(\phi \circ \fss, \psi \circ \gss)$. However, from the first claim, restricting to such forms does not affect the optimality. Therefore, it is without loss of generality to restrict to such forms. Once we have this restriction, the \dname{} minimization algorithm is also a dependence learning algorithm. Moreover, we can always convert a \dname{} into a \rdname{} by introducing regularization terms. One example is demonstrated as follows. We omit its proof.

\begin{property}
  \label{propty:ell2}
  If $\Gamma$ is a \dname, then for all $\lambda > 0, \mu > 0$, $\left(\Gamma(f, g) + \lambda\cdot \E{\norm{f(X)}^2} + \mu\cdot\E{\norm{g(Y)}^2}\right)$ is a \rdname.
\end{property}

We can interpret the regularization terms in \proptyref{propty:ell2} as a weight decay \cite{hastie2009elements} on feature representations, which can be implicitly introduced in optimization, e.g., stochastic gradient descent. Due to these connections, we will consider general \dnames{} in the discussions that follow, and
restrict to the optimal solutions of the form $(\phi \circ \fss, \psi \circ \gss)$. Specifically, we introduce several useful properties for constructing \dnames{}. The proofs can be obtained by applying Jensen's inequality and noting the Markov relation $X-S-T-Y$. %

\begin{property}
  \label{propty:cvx}
  Suppose $\gamma \colon \mathbb{R}^k \times \mathbb{R}^k \to \mathbb{R}$ is  convex with respect to both arguments, then $(f, g) \mapsto \Ed{P_{X, Y}}{\gamma(f(X), g(Y))}$ and $(f, g) \mapsto \Ed{P_{X}P_{Y}}{{\gamma(f(X), g(Y))}}$ are \dnames.
\end{property}

\begin{property}
  \label{propty:loss:e}
  Suppose $\Gamma$ is a \dname{} on $(f, g) \in \spcn{\cX}{k}\times \spcn{\cY}{k}$, and $\nu \colon \mathbb{R} \times  \mathbb{R}^k \times \mathbb{R}^k \times \mathbb{R}^{k \times k}  \to \mathbb{R} $ is nondecreasing with respect to the first argument, then $\nu\left(\Gamma(f, g), \E{f(X)}, \E{g(Y)}, \La_{f, g}\right)$ is also a \dname{}, where $\La_{f, g} = \E{f(X)g^\T(Y)}$.
\end{property}

Given $l$ losses, we can obtain a new loss by using a mapping $\mathbb{R}^l \to \mathbb{R}$ that aggregates the losses. We refer to the operation as a monotonic aggregation if the mapping is nondecreasing with respect to each argument. %

\begin{property}
  \label{propty:agg}  
  A monotonic aggregation of \dnames{} is a \dname.
\end{property}

\subsection{\dnames{} in Learning Practices}%

We demonstrate the connection between \dnames{} and learning practices, including loss functions and learning techniques such as regularization. For convenience, we adopt extended-value extension [cf. \cite[Section 3.1.2]{boyd2004convex}], which allows the value of $\Gamma$ to be infinity. In particular, we define the characteristic function (0-infinity indicator function) of a set $\cD \subset \mathbb{R}^m$, such that for all $\vec{v} \in \mathbb{R}^m$,
\begin{align}
  \id_{\cD}(\vec{v}) \defeq
  \begin{cases}
    0, & \vec{v} \in \cD\\
    \infty,& \vec{v} \notin \cD
  \end{cases}.
  \label{eq:inf}
\end{align}

\subsubsection{Log Loss (Softmax \& Cross Entropy)} 
\label{sec:d:log-loss}We demonstrate the log loss minimization problem is equivalent to a \dname{} minimization problem. Our development will use the following fact.    %
\begin{fact}
\label{fact:log}
The functional $\Ed{P_X}{\log \Ed{P_Y}{\exp(f^\T(X)g(Y))}}$ defined on $(f, g) \in \spcn{\cX}{k} \times \spcn{\cY}{k}$ is a \dname, where the inner expectation is taken w.r.t. $Y \sim P_Y$, and the outer expectation is taken w.r.t. $X \sim P_X$. 
\end{fact}

In particular, let $X$ and $Y$ denote the data variable and categorical label, respectively. We use $h(x) \in \mathbb{R}^d$ to denote the feature representation, which corresponds to the feature in the last hidden layer, and denote the weight and bias term corresponding to $Y = y$ as $w(y) \in \mathbb{R}^d$ and $b(y) \in \mathbb{R}$, respectively. By applying the softmax function, we get a posterior parameterized by $h, w, b$:
\begin{align}
  \Pt^{(h, w, b)}_{Y|X}(y|x) = \frac{\exp(h^\T(x) w(y) + b(y))}{\sum_{y' \in \cY} \exp(h^\T(x) w(y') + b(y'))}.
  \label{eq:Pt}
\end{align}
Then, the log loss is given by the expected value of $-\log \Pt^{(h, w, b)}_{Y|X}$, i.e.,
 $\Ed{(\Xh, \Yh) \sim P_{X, Y}}{ -\log \Pt^{(h, w, b)}_{Y|X}(\Yh|\Xh)}$. %
Let $k = d + 1$, and we define the calibrated bias term $\beta(y) \defeq b(y) - \log P_Y(y)$, and features
\begin{align}
 f =
  \begin{bmatrix}
    1\\ h
  \end{bmatrix} \in \spcn{\cX}{k},\quad
  g =
  \begin{bmatrix}
    \beta\\ w
  \end{bmatrix} \in \spcn{\cY}{k}.
  \label{eq:fg:hw}
\end{align}
Then, we can rewrite the log loss as
$ - \E{f^\T(X)g(Y)}  %
                          + \Ed{P_X}{\log \Ed{P_Y}{\exp(f^\T(X)g(Y))}} + H(Y)$, where $H(Y)$ denotes the entropy of $P_Y$. Therefore, minimizing log loss is equivalent to the minimization of the following functional defined on $(f, g) \in \spcn{\cX}{k} \times \spcn{\cY}{k}$:
\begin{align}
&- \E{f^\T(X)g(Y)}
  + \Ed{P_X}{\log \Ed{P_Y}{\exp(f^\T(X)g(Y))}}\notag\\
  &\qquad+ \E{\id_{\{1\}}(f_1(X))},\label{eq:log:ext}
\end{align}
where $f_1$ denotes the first dimension of $f$, and where $\id$ is the characteristic function as defined in \eqref{eq:inf}. Note that $\E{\id_{\{1\}}(f_1(X))} < \infty$ if and only if $\prob{f_1(X) = 1} = 1$, which guarantees that
the optimal $f$ takes the form of \eqref{eq:fg:hw}.

From \factref{fact:log} and \proptyref{propty:cvx}, \proptyref{propty:loss:e}, \proptyref{propty:agg}, 
the extended log loss \eqref{eq:log:ext} is a \dname. Therefore, for a properly regularized log loss (e.g., \proptyref{propty:ell2}), the optimal $h(X)$ and $w(Y)$ are dependence induced. It is worth mentioning that though $b(Y)$ is not dependence induced in general, the calibrated bias $\beta(Y)$ is dependence induced.

\subsubsection{Support Vector Machine \cite{cortes1995support}} We then consider the loss applied in support vector machine (SVM). In particular, we consider a binary classification problem on data $X$, where the label $Y \in  \cY = \{-1, 1\}$ is balanced. Suppose our goal is to find the optimal feature $h(x) \in \mathbb{R}^d$, weight $\vec{w} \in \mathbb{R}^d$, and bias $b \in \mathbb{R}$. The SVM loss for the separating hyperplane $\ip{\vec{w}}{h(x)} + b = 0$ can be written as \cite{hastie2009elements}
  \begin{align}
    \Ed{P_{X, Y}}{ \ellh(Y, \ip{\vec{w}}{h(X)} + b)} + \lambda \cdot \norm{\vec{w}}^2,
    \label{eq:svm}
  \end{align}
     where  $\ellh \colon \cY \times \mathbb{R} \to \mathbb{R}$ denotes the hinge loss, defined as $ \ellh(y, z) \defeq (1 - yz)^+ $ with $x^+ \defeq \max\{0, x\}$, and where $\lambda > 0$ is a hyperparameter. Let $k = d + 1$ and define feature functions $f \in \spcn{\cX}{k}, g \in \spcn{\cY}{k}$ as
     \begin{align}
       \label{eq:fg:svm}
       f(x) =
       \begin{bmatrix}
         h(x)\\1
       \end{bmatrix},
       \quad
       g(y) = y  \cdot
       \begin{bmatrix}
         \vec{w}\\
         b
       \end{bmatrix}.
     \end{align}
     Then, we have $\ellh(Y, \ip{\vec{w}}{h(X)} + b) = (1 - \ip{\vec{w}Y}{h(X)} - bY)^+ = (1 - \ip{f(X)}{g(Y)})^+$ and
 $\norm{\vec{w}}^2 = \E{\norm{g_{[d]}(Y)}^2}$ where 
we have defined $g_{[d]}\defeq (g_1, \dots, g_d)^\T$ and thus $g_{[d]}(Y) =  Y \cdot \vec{w}$. 

Therefore, the minimization of \eqref{eq:svm} is equivalent to minimizing the functional
\begin{align}
  &\E{\bigl(1 - \fip{f(X)}{g(Y)}\bigr)^+}  + \lambda \cdot \E{\norm{g_{[d]}(Y)}^2} \notag\\
  &\qquad\quad+  \E{\id_{\{1\}}(f_{k}(X))} + \id_{\{0\}}(\E{g(Y)}),
    \label{eq:svm:ext}
\end{align}
where $f_{k}(X)$ denotes the $k$-th dimension of $f$. To see this, note that $\E{\id_{\{1\}}(f_{k}(X))} < \infty$ if and only if $\prob{f_{k}(X) = 1} = 1$, and $\id_{\{0\}}(\E{g(Y)}) < \infty$ if and only if $\E{g(Y)} = 0$, i.e., $g(-1) = -g(1)$. As a result, the optimal solution $(f, g)$ that minimizes \eqref{eq:svm:ext} must take the form \eqref{eq:fg:svm}. Note that from  \proptyref{propty:cvx}, the first three terms of \eqref{eq:svm:ext} are \dnames{}. Then, by combining
\proptyref{propty:loss:e} and \proptyref{propty:agg}, we can verify \eqref{eq:svm:ext} is a \dname.

\subsubsection{Variational Forms of $\varphi$-Divergences} From \proptyref{propty:cvx}, $\left(-\Ed{P_{X, Y}}{f^\T(X)g(Y)} + \Ed{P_XP_Y}{u(f^\T(X)g(Y))}\right)$ is a \dname{} for all convex function $u \colon \mathbb{R} \to \mathbb{R}$. Such functionals appear in characterizing variational forms of $\varphi$-divergences \cite{nguyen2010estimating}, where $u = 
\varphi^\star$ is the convex conjugate of some convex function $\varphi \colon [0, \infty) \to \mathbb{R}$ with $\varphi(1) = 0$.

\subsubsection{Learning Techniques: Regularization, Constrained Features, and Feature Nesting}\label{sec:d:regularize}
 We consider several learning techniques that preserve the \dname{} structure. We first consider the regularization. Given a \dname{} $L(f, g)$, we can construct the regularized loss $L(f, g) + \lambda \cdot R(f, g)$, where $R(f, g)$ is the regularization term and $\lambda > 0$ is the hyperparameter. Then, from \proptyref{propty:agg}, the regularized loss is a \dname{} if $R(f, g)$ is a \dname. From \proptyref{propty:cvx}, the
 squared distance $\E{\norm{f(X) - g(Y)}^2}$, and generally $\E{\norm{f(X) - g(Y)}_p}$ for $p \geq 1$ are \dnames, where $\norm{\vec{v}}_p\defeq \left(\sum_{i} v_i^p\right)^{\frac{1}{p}}$ denotes the $p$-norm. 
Similarly, $\E{\norm{f(X)}_p}$, $\E{\norm{g(Y)}_p}$, $p \geq 1$ and $\E{\norm{f(X)}^2}$, $\E{\norm{g(Y)}^2}$ are also \dnames{}. During implementation, the expectation is estimated by the corresponding empirical averages over each mini-batch.

In contrast to the batch-wise regularization terms, another common practice is directly introducing constraints on each instance of feature representations, e.g., force features to be nonnegative. With the characteristic function \eqref{eq:inf}, we can use an extended loss to characterize such constraints. For example, suppose the original loss $L(f, g)$ is a \dname{} and $f(X)$ is constrained to be within a set $\cC \subset \mathbb{R}^k$. Then we can consider the extended loss $L(f, g) + \E{\id_{\cC}(f(X))}$, which leads to the same optimal solution. From \proptyref{propty:cvx}, the extended loss is a \dname{} if $\cC$ is convex. Two important examples of such convex sets are: 1) $p$-norm balls, e.g.,  $\cC = \{\vec{v} \in \mathbb{R}^k\colon \norm{\vec{v}}_p \leq 1\}$, where each feature instance is projected onto the norm ball; 2) nonnegative orthant, i.e., $\cC = \{\vec{v} \in \mathbb{R}^k\colon v_i \geq 0\}$, corresponding to learning nonnegative features.

The regularization methods are also used to combine multiple training objectives by optimizing their weighted sum. From \proptyref{propty:agg}, if each objective is a \dname, the weighted sum will also be a \dname. A special example is \emph{Feature Nesting}, which constructs the weighted sum of the same loss applied on a set of nested features. Specifically, suppose the loss $L$ is defined for all $d$-dimensional feature function pair $(f, g)$, e.g., $\E{f^\T(X)g(Y)}$. Then, give $(f, g) \in \spcn{\cX}{k} \times \spcn{\cY}{k}$, we consider the nested loss 
\begin{align}
  \sum_{i =1 }^k c_i \cdot  L(f_{[i]}, g_{[i]}),\label{eq:nest}  
\end{align}
 where  $c_i, i = 1, \dots, k$ are non-negative weights, and where for each $i \in [k]$, we have defined $f_{[i]} \in \spcn{\cX}{i}, g_{[i]} \in \spcn{\cY}{i}$ as
\begin{align}
f_{[i]} \defeq (f_1, \dots, f_i)^\T, \qquad
g_{[i]} \defeq (g_1, \dots, g_i)^\T.\label{eq:fg:i}  
\end{align}
 Such construction is applied to induce structures on the representations $f, g$, referred to as the \emph{nesting technique} \cite[Section 4]{xu2024neural}, also called \emph{matryoshka representation learning} \cite{kusupati2022matryoshka}. The nested loss \eqref{eq:nest} is a \dname{} if $L$ is a \dname.

\subsection{Representations Induced by Strict Dependence} %

We conclude this section by discussing a special case where $X$ and $Y$ are strictly dependent, and we have the following characterization [cf. \corolref{cor:ss:extreme}]. 
\begin{proposition}
  \label{prop:sd}
  Given a \rdname{} $L(f, g; P_{X, Y})$ defined for $f \in \spcn{\cX}{k}, g \in \spcn{\cY}{k}$, let us define the representation learning algorithm $\alg\colon (X, Y) \mapsto \argmin_{f, g}L(f, g; P_{X,Y})$. Suppose $Y = \ly(X)$ for a mapping $\ly \colon \cX \to \cY$. Then, we have
  \begin{align}
    \alg(X, Y) = \{(\phi \circ \ly, \psi) \colon (\phi, \psi) \in \alg(Y, Y)\}.
    \label{eq:nc1}
  \end{align}
\end{proposition}

 Note that a specific case of \propref{prop:sd} is the classification scenario, where $Y$ corresponds to the label of data $X$. The mapping $\ly$ 
corresponds to the labeling function, and $Y = \ly(X)$ indicates that we can obtain such a label from data without any ambiguities. This happens when $P_{X, Y}$ is the empirical distribution of a classification dataset $\{(x_i, y_i)\}_{i = 1}^n$ with distinct $X$ samples, i.e., $x_i \neq x_j$ for all $i \neq j$. For such datasets, from \eqref{eq:nc1}, the optimal feature function $f$ satisfies $f(x) = \phi(\ly(x))$, i.e., all $X$ samples with the same label $\ly(X)$ share the same representation.

This phenomenon has been observed in training classification deep neural networks with cross entropy \cite{papyan2020prevalence}, referred to as the ``neural collapse''\footnote{The Neural Collapse (NC) \cite{papyan2020prevalence} was defined as the collection of four interconnected phenomena, NC1-NC4, where NC1 corresponds to the phenomenon of $f(X)$ being an abstraction of $Y$, also referred to as the ``variability collapse.''}, where theoretical characterizations were discussed therein and also in subsequent works \cite{mixon2022neural, zhu2021geometric}. Our development demonstrates that this phenomenon is essentially a consequence of the special dependence structure of training data and is shared by a large collection of losses.

\optional{in addition, if balanced, $\rank = K - 1$, $P_{S,T}$ takes nice form. symmetry of the loss implies symmetry of representations;}

\optional{applications in computing dependence induced quantities, e.g., mutual info, ..., }

\section{Learning With Feature Adapters} %

The composition structure of dependence induced representations provides a structured implementation of dependence learning algorithms, where we use $S = \fss(X), T = \gss(Y)$ as an intermediate interface for feature representations.
As shown in \figref{fig:nn:adp}, when $S, T$ are given, the minimization of \dname{} $L(f, g; P_{X, Y})$ can be solved by minimizing $L(\phi, \psi; P_{S, T})$ over $(\phi, \psi)$, and we can retrieve the optimal representations via $f(X) = \phi(S), g(Y) = \psi(T)$ [cf. \thmref{thm:sep:loss}]. We call such $\phi, \psi$ as \emph{feature adapters}. Therefore, we can obtain solutions to different \dnames{} by training only the feature adapters $\phi, \psi$, without changing $\fss, \gss$. The maximal correlation functions $\fss, \gss$ can be learned separately, e.g., as pre-trained deep neural networks, where the training details will be discussed later.

\begin{figure}[!th]
  \centering
  \resizebox{.34\textwidth}{!}{
\begin{tikzpicture}[auto, thick, node distance=2cm, >=latex'
  ]
  \def\dy{1.3}

  \draw node [block] (L) at (5.5, -3 * \dy)  {\Large $L(\phi, \psi)$};

  \foreach \i/\ti/\tu/\to/\loc in {1/$X$/$S$/$f(X)$/above, 2/$Y$/$T$/$g(Y)$/below}
  {
    \draw node [ext, frozenfill, rotate = -90]  (f-\i) at (0, -2 * \dy * \i cm){};
    \draw node [ext, right of = f-\i, node distance = 2.4cm, rotate = -90]  (phi-\i) {};
    \draw node [input, left of = f-\i] (x-\i) {}
    node [xshift = -0.5cm] at (x-\i) {\Large \ti}
    node [xshift = -0.6mm] at (x-\i) {\Large $\circ$};
    \draw [->] (x-\i) -- (f-\i);
    \draw [->] (f-\i)-- (phi-\i) node [\loc, midway] {\Large \tu};
    \draw [->] (phi-\i) -| (L) node [\loc, pos = .18] {\large\to};
  }


  \draw node at (f-1)  {\Large $\fss$}
  node at (f-2)  {\Large $\gss$}
  node at (phi-1)  {\Large $\phi$}
  node at (phi-2)  {\Large $\psi$};


\end{tikzpicture}

}
  \caption{Feature learning by training adapters $\phi, \psi$. The feature extractors $\fss$, $\gss$ can be frozen pre-trained networks or non-trainable modules.
  }
  \label{fig:nn:adp}
\end{figure}
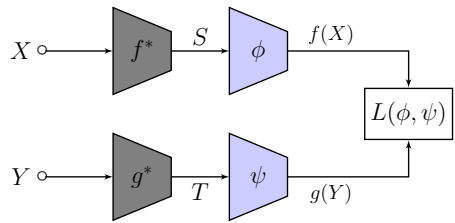
The separation enables more efficient implementation of the learning procedures.
In particular, when $S, T$ have much simpler structures compared with original $X, Y$, we can use lightweight neural networks as the adapters $\phi, \psi$, and the adapter training does not require back-propagating learning errors  \cite{rumelhart1986learning} to the $\fss, \gss$ modules. %

 Note that the representation interface $S, T$ are both informationally and computationally efficient: from the minimal sufficiency, $S, T$ contain necessary information without introducing redundancies; even without adaptations,  we can use $\fss, \gss$ to conduct inference tasks by some linear assembling processes, as discussed in \cite{xu2024neural}.

\subsection{Feature Adapters for Dependence Learning}%

The feature adapters allow us to adapt to different \dnames{} by training only the adapters. In particular, we can use feature adapters to implement constrained feature learning and inference-time hyperparameter tuning.
\subsubsection{Constrained Feature Learning With Adapters} When the training loss is a \dname, we can implement constrained feature learning with the architectures shown in \figref{fig:nn:adp}. Specifically, from our discussions in \secref{sec:d:regularize}, we can effectively train the adapters $\phi, \psi$ to learn features restricted to $p$-norm balls or non-negative features $f$, $g$.

\subsubsection{Inference-time Hyperparameter Tuning} Suppose we have a class of \dnames{} $L_{\lambda}(f, g; P_{X, Y})$, where $\lambda \in \cI$ is the hyperparameter, e.g., the weights for regularization terms. The optimal feature functions that minimizes $L_{\lambda}$ depend on $\lambda$, which we denote by $(f^{(\lambda)}, g^{(\lambda)})$. The values of $\lambda$ are often determined by comparing performances of different $\lambda$, which requires learning $(f^{(\lambda)}, g^{(\lambda)})$ for a set of different values. 

\begin{figure}[!th]
  \centering
  \resizebox{.34\textwidth}{!}{
\begin{tikzpicture}[auto, thick, node distance=2cm, >=latex'
  ]
  \def\dy{1.3}

  \draw node [block] (L) at (5.5, -3 * \dy)  {\Large $L(\phi, \psi)$};

  \foreach \i/\ti/\tu/\to/\loc in {1/$X$/$S$/$f^{(\lambda)}$/above, 2/$Y$/$T$/$g^{(\lambda)}$/below}
  {
    \draw node [ext, frozenfill, rotate = -90]  (f-\i) at (0, -2 * \dy * \i cm){};
    \draw node [ext, right of = f-\i, node distance = 2.4cm, rotate = -90]  (phi-\i) {};
    \draw node [input, left of = f-\i] (x-\i) {}
    node [xshift = -0.5cm] at (x-\i) {\Large \ti}
    node [xshift = -0.6mm] at (x-\i) {\Large $\circ$};
    \draw [->] (x-\i) -- (f-\i);
    \draw [->] (f-\i)-- (phi-\i) node [\loc, midway] {};
    \draw [->] (phi-\i) -| (L) node [\loc, pos = .18] {\large\to};
  }


  \draw node at (f-1)  {\Large $\fss$}
  node at (f-2)  {\Large $\gss$}
  node at (phi-1)  {\Large $\phi^{(\lambda)}$}
  node at (phi-2)  {\Large $\psi^{(\lambda)}$};

  \draw node (l) at ($(phi-1)!.5!(phi-2)$) {\large $\lambda$};
  \draw [->] (l) -- (phi-1);
  \draw [->] (l) -- (phi-2);


\end{tikzpicture}

}
  \caption{The feature adapters $\phi^{(\lambda)}, \psi^{(\lambda)}$ are parameterized by the hyperparameter $\lambda$, which is tunable during inference.} %
  \label{fig:nn:lambda}
  \end{figure}
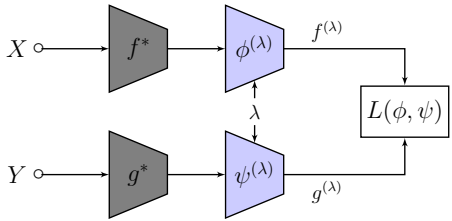

From \thmref{thm:sep:loss}, we have
$(f^{(\lambda)}, g^{(\lambda)}) = (\phi^{(\lambda)} \circ \fss, \psi^{(\lambda)} \circ \gss)$. 
Note that only the adapters $(\phi^{(\lambda)}, \psi^{(\lambda)})$ depend on the $\lambda$, and can be learned by minimizing $L_{\lambda}(\phi^{(\lambda)}, \psi^{(\lambda)}; P_{S, T})$. Therefore, it suffices to learn feature adapters parameterized by $\lambda$. Specifically, we define the loss 
$L(\phi, \psi; P_{S, T}) \defeq \Ed{\lambda \sim P_{\lambda}}{L_{\lambda}(\phi^{(\lambda)}, \psi^{(\lambda)}; P_{S, T})}$,
 where $P_\lambda$ is a distribution supported on $\cI$. 
Then, minimizing $L(\phi, \psi; P_{S, T})$ leads to parameterized $\phi^{(\lambda)}, \psi^{(\lambda)}$, which give parameterized representations $f^{(\lambda)}(X), g^{(\lambda)}(Y)$, as shown in \figref{fig:nn:lambda}.

In contrast to the common practice, this design allows us to tune the hyperparameter $\lambda$ over a continuous range, without retraining or requiring validation samples. 
This can be useful in learning scenarios where we have many different downstream tasks but few validation samples.

\subsection{Learning Maximal Correlation Functions }

The maximal correlation functions can be effectively learned from data via maximizing the nested H-score \cite{xu2024neural}. Specifically, for $f \in \spcn{\cX}{k}$, $g \in \spcn{\cY}{k}$, we define the nested H-score $\Hnest(f, g)$ as\footnote{
 The form of the nested H-score is nonunique; for example, the weights before each H-score can be arbitrary positive numbers. See \cite[Section 4.1]{xu2024neural} for detailed discussions.
 }
 \begin{align}
  \Hnest(f, g) \defeq \sum_{i = 1}^k \Hs(f_{[i]}, g_{[i]}),
  \label{eq:Hnest:def}
\end{align}
where $f_{[i]}, g_{[i]}$ are as defined in \eqref{eq:fg:i}, and where
 $\Hs(f, g)$ is the H-score, defined as
\begin{align}%
  \Hs(f, g)
  &\defeq \E{f^{\T}(X) g(Y)}
    - \left(\E{f(X)}\right)^{\T}\E{g(Y)}\notag\\
   &\qquad - \frac12  \cdot \tr\left(\Lambda_{f}\Lambda_{g}\right),
     \label{eq:H:def}
\end{align}
where $\La_f = \E{f(X)f^\T(X)}$, $\La_g = \E{g(Y)g^\T(Y)}$. It can be verified %
that $-\Hnest(f, g)$ is a \rdname. 

Note that since $\fss(X), \gss(Y)$ are minimal sufficient statistics, it is possible to learn $\fss(X), \gss(Y)$ by adapting existing sufficient statistics, which can significantly reduce the computation costs. Such sufficient statistics can be obtained from pretrained models, e.g., deep classifiers. In engineering applications, we can also obtain analytical forms of sufficient statistics from corresponding physical system models.

\optional{remark1: also worth mention that generally;
 it is possible to replace $\fss$, $\gss$ by the first few dimensions, as long as they are informationally equivalent. 
 }

\optional{why we need the stochastic mapping}

\bibliographystyle{IEEEtran}
\bibliography{ref}

\newpage

\optional{ss as partition; coarser partition, merging;}

\optional{optimization problem:
  \begin{align*}
    &\minimize_{f\colon \cX \to \cS, h \colon \cS \times \cY \to \mathbb{R} } |\cS|\\
    &\st h(f(x), y) = \lpmi_{X; Y}(x, y)
  \end{align*}
  \begin{align*}
    &\minimize_{f\colon \cX \to \cS, g \colon \cY \to \cT, h \colon \cS \times \cT \to \mathbb{R} } |\cS| + |\cT|\\
    &\st h(f(x), g(y)) = \lpmi_{X; Y}(x, y)
  \end{align*}
  \begin{align*}
    &\maximize_{f\colon \cX \to \cS, g \colon \cY \to \cS} |\cS|\\
    &\st f(x) = g(y)
  \end{align*}
}

\optional{ (continuous + unique solution + one-to-one mapping, -> dependence induced), no $Z, W$ needed }

\optional{ 
  \begin{itemize}
  \item extensions: single sided version.
  \item operations preserve the property, optimizing over a (matrix) scaling factor, shift, ....
\end{itemize}
}

\optional{ \xedit{operator monotone function}, on $\La_f$, $\La_g$ ; }

\end{document}